\documentclass[10pt]{article}

\usepackage{mlpackage}
\usepackage[backend=biber,
            style=authoryear,
            sorting=nyt,
            uniquename=true
           ]{biblatex}
\providetoggle{arxiv}
\settoggle{arxiv}{true}

\DeclareCiteCommand{\cite}
  {\usebibmacro{prenote}}
  {\usebibmacro{citeindex}%
   \printtext[bibhyperref]{\usebibmacro{cite}}}
  {\multicitedelim}
  {\usebibmacro{postnote}}

\DeclareCiteCommand*{\cite}
  {\usebibmacro{prenote}}
  {\usebibmacro{citeindex}%
   \printtext[bibhyperref]{\usebibmacro{citeyear}}}
  {\multicitedelim}
  {\usebibmacro{postnote}}

\DeclareCiteCommand{\parencite}[\mkbibparens]
  {\usebibmacro{prenote}}
  {\usebibmacro{citeindex}%
    \printtext[bibhyperref]{\usebibmacro{cite}}}
  {\multicitedelim}
  {\usebibmacro{postnote}}

\DeclareCiteCommand*{\parencite}[\mkbibparens]
  {\usebibmacro{prenote}}
  {\usebibmacro{citeindex}%
    \printtext[bibhyperref]{\usebibmacro{citeyear}}}
  {\multicitedelim}
  {\usebibmacro{postnote}}

\DeclareCiteCommand{\footcite}[\mkbibfootnote]
  {\usebibmacro{prenote}}
  {\usebibmacro{citeindex}%
  \printtext[bibhyperref]{ \usebibmacro{cite}}}
  {\multicitedelim}
  {\usebibmacro{postnote}}

\DeclareCiteCommand{\footcitetext}[\mkbibfootnotetext]
  {\usebibmacro{prenote}}
  {\usebibmacro{citeindex}%
   \printtext[bibhyperref]{\usebibmacro{cite}}}
  {\multicitedelim}
  {\usebibmacro{postnote}}

\DeclareCiteCommand{\textcite}
  {\boolfalse{cbx:parens}}
  {\usebibmacro{citeindex}%
   \printtext[bibhyperref]{\usebibmacro{textcite}}}
  {\ifbool{cbx:parens}
     {\bibcloseparen\global\boolfalse{cbx:parens}}
     {}%
   \multicitedelim}
  {\usebibmacro{textcite:postnote}}

\usepackage{tikz}			
\usetikzlibrary{arrows} 	

\tikzset{->, >=stealth', shorten >=1pt, auto, node distance=2cm, semithick, baseline=(current bounding box.center)}

\DeclareMathOperator{\sspan}{span}                       
\DeclareMathOperator{\relu}{relu}                       
\let\Im\relax                                           
\DeclareMathOperator{\Im}{Im}                           
\let\Re\relax
\DeclareMathOperator{\Re}{Re}                           

\newcommand\transaux{\intercal}                         
\newcommand\trans[1]{#1^\transaux}                      
\newcommand\transc[1]{#1^{\textup{H}}}                    

\newcommand\I{\mathrm{I}}                               

\newcommand*\rnn{\textsc{rnn}}

\newcommand*\mnist{\textsc{mnist}}

\newcommand*\ad{\textsc{ad}}
\newcommand*\spd{\textsc{spd}}

\DeclareMathOperator{\GLaux}{GL}                         

\NewDocumentCommand{\GL}{ m }{ \GLaux\pa{#1} }

\NewDocumentCommand{\M}{ >{\SplitArgument{1}{,}}m}{%
    \RR^{\prodaux #1}%
}
\NewDocumentCommand{\MC}{ >{\SplitArgument{1}{,}}m}{%
    \CC^{\prodaux #1}%
}

\NewDocumentCommand{\commasaux}{ m m }{%
    \IfNoValueTF{#2}{ #1 }{ #1, #2 }%
}
\NewDocumentCommand{\prodaux}{ m m }{%
    \IfNoValueTF{#2}{ #1 \times #1 }{ #1 \times #2 }%
}

\NewDocumentCommand{\liebrack}{s O{} >{\SplitArgument{1}{,}}m}{%
    \IfBooleanTF{#1}{\liebrackaux*#3}{\liebrackaux[#2]#3}%
}
\DeclarePairedDelimiterX{\liebrackaux}[2]{\lbrack}{\rbrack}{#1, #2}

\NewDocumentCommand{\scalarother}{s O{} >{\SplitArgument{1}{,}}m}{%
    \IfBooleanTF{#1}{\scalarotheraux*#3}{\scalarotheraux[#2]#3}%
}
\DeclarePairedDelimiterX{\scalarotheraux}[2]{\ll}{\gg}{#1, #2}

\title{Automatic Differentiation: Theory and Practice}
\author{Mario Lezcano-Casado\thanks{Senior Software Engineer at Quansight
\newline \hspace*{4.4ex}Correspondence: mlezcano@quansight.com}}
\date{}

\begin{document}
\maketitle

\tableofcontents

\begin{abstract}
    We present the classical coordinate-free formalism for forward and backward mode \ad{} in the real and complex setting.
    We show how to formally derive the forward and backward formulae for a number of matrix functions starting from basic principles.
\end{abstract}

This document is aimed at a reader who has some experience with auto-differentiation (\ad{}) frameworks (PyTorch, Tensorflow, JAX\textellipsis) and would want to get a better understanding of the mathematical ideas behind their forward and backward \ad{} engines. We expect that the reader will be able to derive from scratch forward and backward formulae for virtually any function in PyTorch, Tensorflow, or JAX after reading these notes.

We do not assume any previous mathematical knowledge, but some mathematical maturity, as we will present the necessary results from different areas of mathematics in the text. That said, some familiarity with linear algebra and multivariate calculus would be helpful.

There will be a plethora of examples throughout the text to make the abstract definitions easier to digest. We will omit some computations in the examples, but these will always be simple enough so that the reader should be able to deduce them themselves after some thinking.

\Cref{sec:linear_algebra,sec:multivariable_calculus} contain a review of a number of definitions and results from linear algebra and multivariate calculus. We encourage the reader that already feels comfortable with this material to have a quick read through these, as we will use them to set the notation.

\Cref{sec:forward_mode,sec:backward_mode} show how to use the definitions from the first two sections to derive the forward and backward (\ad) formulae for most functions used in machine learning or statistics.

\Cref{sec:complex} treats the more advance topic of differentiation of functions of complex variable. It may be skipped on a first read.

\clearpage
\section{Coordinate-free Linear Algebra}\label{sec:linear_algebra}
In this section we answer the question: What is a coordinate-free approach and why do we care? With this objective in mind, we define abstract vector spaces, bases, and linear maps.
A secondary objective of this section is to show that thinking of a linear map as something that preserves the operations of a vector space is easier than thinking about it as a matrix.

\subsection{Vector spaces}
\begin{definition}\label{def:vector_space}
    A \textbf{vector space} over a field $\KK$ (think $\KK = \RR$ or $\KK = \CC$) is a set $V$ with two operations $\deffun{+ : V \times V -> V;}$ and $\deffun{\cdot : \KK \times V -> V;}$ such that for every $u, v, w \in V$ and $a, b \in \KK$

    \vspace{\baselineskip}
    \noindent
    \begin{minipage}[t]{.5\textwidth}
    \begin{itemize}
        \item Associativity: $(u + v) + w = u + (v + w)$
        \item Identity: There exists a $0 \in V$ such that \\\phantom{Identity: }$\forall v \in V$, $v + 0 = v$
        \item Inverse: For every $v \in V$ there is a $-v \in V$\\\phantom{Inverse: }such that $v + (-v) = 0$
        \item Commutativity: $u + v = v + u$
    \end{itemize}
    \end{minipage}%
    \begin{minipage}[t]{.5\textwidth}
    \begin{itemize}
        \item Compatibility: $a\cdot (b \cdot v) = (ab)\cdot v$
        \item Scalar identity: $1\cdot v = v$ where $1\in\KK$\\\phantom{Scalar identity: }is the multiplicative identity
        \item Distributivity $1$: $a \cdot (u+v) = a \cdot u + a \cdot v$
        \item Distributivity $2$: $(a+b) \cdot u = a \cdot u + b \cdot u$
    \end{itemize}
    \end{minipage}
    \\[8pt]

    As with any mathematical multiplication, we often omit the symbol and simply write $av = a \cdot v$.

    The elements of $V$ are called \textbf{vectors} and the elements of $\KK$ are called \textbf{scalars}.

    Vector spaces over the real (resp.\ complex) numbers are called \textbf{real} (resp.\ \textbf{complex}) vector spaces.
\end{definition}

The first time one encounters these axioms, they can be quite overwhelming. They become a bit easier to digest once one realises that all they are doing is to model $\RR^n$ abstractly. Once one sits down and checks that $\RR^n$ fulfils all these axioms, we find the first example of a vector space.

\begin{example}
    $\RR^n$ is a vector space over $\RR$.
\end{example}

$\RR^n$ is the most important example of a vector space. In fact, virtually all the vector spaces over the real numbers that we will work with will have an $\RR^n$ lurking behind one way or another. Whenever we think about an abstract vector space, it is good to picture in mind something that looks like $\RR^n$ to fix the ideas. That being said, there are other vector spaces.

\begin{example}\label{ex:complex}
    $\CC^n$ is a vector space over $\CC$.

    $\CC^n$ is also a real vector space, as addition of vectors is performed on the real and complex part separately.
    In fact, $\CC^n$ is pretty much the same as $\RR^{2n} = \RR^n \times \RR^n$ as a real vector space.
\end{example}

We finish with some slightly less conventional examples of vector spaces.

\begin{example}\label{ex:3}
    The set of polynomials in one variable of degree less or equal to $n$ is a real vector space.

    The set of polynomials in one variable of degree exactly equal to $n$ is \textbf{not} a real vector space.

    The set of matrices $\M{m, n}$ is a real vector space.

    The set of tensors with real entries of a fixed shape is a real vector space.

    The set of infinite sequences of real numbers $(a_i)_{i=1}^\infty = (a_1, a_2, \dots)$ is a real vector space.
\end{example}

It is a good exercise to convince oneself that all these examples are what they claim to be. In general, to check that an object is indeed a vector space, it tends to be enough to check that multiplying a vector by a scalar gives a vector in the set and adding two vectors in the set gives a vector in the set (the second example above is a counterexample of this).

\subsection{Vector subspaces and bases}
\begin{definition}
    A subset $U$ of a vector space $V$ over $\KK$ is a \textbf{vector subspace} or \textbf{linear subspace} if it is closed under addition and multiplication by scalars. In symbols,
    \[
        u+v \in U \quad \text{and} \quad av \in U \mathrlap{\qquad \forall u,v \in U, a \in \KK}
    \]
\end{definition}

\begin{example}
    The vectors of the form $(0, a_2, \ldots, a_n)$ with $a_i \in \RR$ form a linear subspace of $\RR^n$.

    The polynomials in one real variable of degree less than $n$ form a linear subspace of the vector space of all polynomials in one real variable.
\end{example}

A simple way to define subspaces is to pick a number of vectors and consider all the possible combinations of additions and multiplication by scalars that one can form---\ie, consider their closure.

\begin{definition}
    The \textbf{vector space spanned by $v_1, \dots, v_n \in V$} is defined as
    \[
    \sspan(v_1, \dots, v_n) = \set[\Big]{\sum_{i=1}^n a_iv_i | a_i \in \KK} \subset V.
    \]
\end{definition}

The vector space spanned by a number of vectors is always a vector subspace of $V$ by definition, as it is closed with respect to sums and product by scalars.

\begin{example}\label{ex:span} We compute the span of some sets of vectors in $V = \RR^3$:
    \begin{itemize}
        \item Let $v_1 = \pa{0, 1, 0}$. $\sspan\pa{v_1} = \set{\pa{0, a, 0} | a \in \RR}$.
        \item Let $v_2 = \pa{0, 1, 1}$. $\sspan\pa{v_1, v_2} = \set{\pa{0, a + b, b} | a, b \in \RR}$. Setting $a = c-b$ for a new variable $c \in \RR$, we have that $\sspan\pa{v_1, v_2} = \set{\pa{0, c, b} | c, b \in \RR}$. In plain words, $v_1$ and $v_2$ generate the subspace of all the vectors in $\RR^3$ with first coordinate equal to zero.
        \item Let $v_3 = \pa{0, 1, 2}$. $\sspan\pa{v_1, v_2, v_3} = \sspan\pa{v_1, v_2}  = \set{\pa{0, a, b} | a, b \in \RR}$. What causes this is that $v_3 = 2v_2 - v_1$. As such, any vector that we can form using $v_3$, we could already form as a linear combination of $v_1$ and $v_2$. Intuitively, the third vector is ``redundant''.
    \end{itemize}
\end{example}

If we have vectors $v_1, \dots, v_n$, and we express another vector as a linear combination of these, $v = \sum_{i=1}^n a_i v_i$, we would like to be able to represent $v$ as $(a_1, \dots, a_n)$, as we do when we are in $\RR^n$. This identification may not be possible because of two reasons.

\begin{enumerate}
    \item \textbf{Our set might not span all the vectors in $V$}. In the examples in~\Cref{ex:span}, none of the sets of vectors that we considered was able to span the vector $v = (1, 0, 0)$. Formally, we write $(1, 0, 0) \not\in \sspan\pa{v_1, v_2, v_3}$. More generally, it might be the case that $\sspan\pa{v_1, \dots, v_n} \subsetneq V$.
    \item \textbf{The representation $v = \sum_{i=1}^n a_iv_i$ might not be unique}. This is what happened in the last example in~\Cref{ex:span}. Since $v_3 = 2v_2 - v_1$, if we choose a vector in $\sspan(v_1, v_2, v_3)$, for example $v = (0, 2, 2)$, we have that we can write
        \[
            v = 2v_2 = v_1 + v_3 = \tfrac{2}{3}v_1 + \tfrac{2}{3}v_2 + \tfrac{2}{3}v_3.
        \]
        Thus, we cannot identify $v = \sum_{i=1}^3 a_i v_i$ with $(a_1, a_2, a_3)$ because the choice of $a_i$ is not unique.
\end{enumerate}

These problems motivate the following definition:
\begin{definition}
    A set of vector $B = \set{v_1, \dots, v_n}$ is called a \textbf{basis} of a vector space $V$ if $\sspan(B) = V$ and any element $v \in V$ may be represented uniquely in terms of the vectors in $B$. If there is a basis of $n$ vectors for $V$, we say that \textbf{$V$ has dimension $n$} or $\dim V = n$.
\end{definition}

In plain words, a basis of a vector space is a set of vectors such that
\begin{itemize}
    \item Spans the whole $V$.
    \item Has no redundancies: No vector in $B$ can be expressed as a linear combination of others.
\end{itemize}

\subsection{Linear maps}
The other object that linear algebra studies are linear maps. These are maps between vector spaces that respect the vector space structure.

\begin{definition}\label{def:linear_map}
    A map $\deffun{T : V -> W;}$ between vector spaces over a field $\KK$ is said to be $\KK$-\textbf{linear}---or simply \textbf{linear}---if
    \[
        T(u+v) = T(u) + T(v) \qquad T(av) = aT(v) \mathrlap{\qquad \forall u,v \in V, a \in \KK.}
    \]
\end{definition}

The following lemma shows that there is a very close connection between linear maps and bases.

\begin{lemma}[A linear map is defined by its values on a basis]\label{lemma:linear_map_basis}
    Let $V$ be a vector space with a basis $B_V = \set{v_1, \dots, v_n}$. If we choose $w_1, \dots, w_n \in W$, there exists a unique linear map $\deffun{T : V -> W;}$ such that $T(v_i) = w_i$.
\end{lemma}
\begin{proof}
     Since $B_V$ is a basis, there exists a unique representation of an arbitrary element $v\in V$ as $v = a_1v_1 +\dots+a_nv_n$.
    We can define $T$ on an arbitrary vector $v$ as
    \[
        T(v) = T(a_1 v_1 + \dots + a_n v_n) = a_1T(v_1) + \dots + a_nT(v_n).
    \]
    so $T$ exists. To show that it is unique, suppose that there is another linear map $T'$ that takes the same values on $B_V$. We then have that they are equal
    \[
        T(v) = a_1T(v_1) + \dots + a_nT(v_n) = a_1T'(v_1) + \dots + a_nT'(v_n) = T'(a_1 v_1 + \dots + a_n v_n) = T'(v)
    \]
    where we have used that $T(v_i) = T'(v_i)$ and the fact that both $T$ and $T'$ are linear.
\end{proof}

The best linear maps are those that map bases to bases.

\begin{definition}\label{def:linear_iso}
    A linear map $\deffun{T : V -> W;}$ is called a \textbf{linear isomorphism} if for any basis $B$ of $V$, $T(B) = \set{T(v_1), \dots, T(v_n)}$ is a basis of $W$.
\end{definition}

This definition packs quite a bit of information. First, it says that $T$ is surjective. This is because, since $T(B)$ is a basis of $W$, $\sspan(T(B)) = W$ and any $w \in W$ is in the image of $T$:
\[
    w = a_1T(v_1) + \dots + a_nT(v_n) = T(a_1v_1 + \dots + a_n v_n).
\]
It also says that this representation of a vector in terms of $T(B)$ is unique, as that is the other property that bases have. In summary, we can represent any vector in $W$ uniquely in terms of $B$ and $T$. Reciprocally, if we denote basis $T(v_i) = w_i$, by~\Cref{lemma:linear_map_basis}, there exists a unique linear map $\deffun{T' : W -> V;}$ such that $T'(w_i) = v_i$. By using the linearity of $T$ and $T'$ It is easy to prove that
\[
    T'(T(v)) = v \quad \text{and}\quad T(T'(w)) = w \mathrlap{\qquad\forall v \in V, w \in W.}
\]
It should be clear now that a better name for this linear map would be $T^{-1}$. In fact, we have proved the following:
\begin{proposition}
    A linear isomorphism as per~\Cref{def:linear_iso} has a linear inverse.\footnote{The reciprocal also holds making this another possible definition of a linear isomorphism.}
\end{proposition}

This is another characterisation of linear isomorphisms which is very useful in practice:
\begin{proposition}[Characterisation of Linear Isomorphisms]\label{prop:characterisation_linear}
    Let $\deffun{T : V -> W;}$ be a linear map. If $T(B)$ is a basis of $W$ for one basis $B$ of $V$, then $T(B')$ is a basis of $W$ for any basis $B'$ of $V$, that is, $T$ is a linear isomorphism.
\end{proposition}

Linear isomorphisms are the nicest maps in linear algebra, as they allow us to translate computations on one space to computations on another and back.

\begin{example}\label{example:polynomials}
    A basis of the set of polynomials $\mathcal{P}_{<n}$ with real coefficients of degree less than $n$ is given by $B = \set{1, x, x^2, \dots, x^{n-1}}$, as any polynomial $p \in \mathcal{P}_{<n}$ is represented as a linear combination of these as $p(x) = \sum a_i x^i$. Denoting $p_i(x) = x^i$, we may map this basis of $\mathcal{P}_{<n}$ into the canonical basis of $\RR^n$ via
    \[
        T(p_0) = (1, 0, \dots, 0) \quad
        T(p_1) = (0, 1, \dots, 0) \quad
        \dots \quad
        T(p_{n-1}) = (0, 0, \dots, 1)
    \]
    This extends to a linear map $\deffun{T: \mathcal{P}_{<n} -> \RR^n;}$ by~\Cref{lemma:linear_map_basis} and since it maps a basis to a basis, by~\Cref{prop:characterisation_linear}, it is a linear isomorphism. In other words, $\mathcal{P}_{<n}$ and $\RR^n$ are \emph{isomorphic} as real vector spaces.
\end{example}

\subsection{Differences between a real vector space and $\RR^n$}\label{sec:rn}
For now, we have set up the basic ideas of linear algebra: We have a vector space, which is an abstract space in which we can add vectors and multiply vectors by scalars, we can encode most of the information of a vector space into a basis, and we can map vector spaces to other vector spaces via linear maps. In particular, if we have a linear isomorphism between vector spaces, this allows us to translate operations on one to operations on the other and back, rendering them somewhat equivalent.

On the other hand, while doing all this, we have barely talked about $\RR^n$. This is a bit odd given that it is the main space that we want to study. The only thing that we have mentioned is that ``real vector spaces model $\RR^n$'' and that, as one could expect, $\RR^n$ is a real vector space. But why would we care about real vector spaces in the abstract? Why not simply work on $\RR^n$?

There are two main differences between real vector spaces and $\RR^n$. One huge and one subtler.

The huge one is that a general vector space may have infinite dimensions. We put an example of such a vector space in~\Cref{ex:3}, when we talked about infinite sequences of real numbers. Infinite-dimensional vector spaces are whole different beasts, and we will not talk about them here.

Now, a real vector space $V$ of dimension $n$, it looks much more similar to $\RR^n$ as
\begin{theorem}\label{thm:RnisoV}
    Let $V$ be a real vector space of dimension $n$, there exists a linear isomorphism $\deffun{T : V -> \RR^n;}$.\footnote{The real numbers are not important here. A vector space of dimension $n$ over $\KK$ is isomorphic to $\KK^n$ using the same argument.}
\end{theorem}
\begin{proof}
    Choose a basis $B = \set{v_1, \dots, v_n}$ of $V$ and perform the same construction that we did in~\Cref{example:polynomials} constructing $T$ by sending $T(v_i) = e_i$.
\end{proof}

This isomorphism says that we can identify every element from $V$ with a vector from $\RR^n$. As such, if we need to do computations in $V$, we can map $V$ into $\RR^n$ by $T$, do the computations on $\RR^n$ and map the result back to $V$ via $T^{-1}$, potentially simplifying the abstract space to working with $\RR^n$.

In the same way, if we have a linear map between vector spaces, and we have bases for these vector spaces, we can write this map in coordinates, giving raise to the following well-known concept.

\begin{definition}\label{def:matrix}
    Given a linear map between real vector spaces $\deffun{T : V -> W;}$ with bases $B_V = \set{v_1, \dots, v_n}, B_W = \set{w_1, \dots, w_m}$, we define the \textbf{matrix associated to $T$} as the element $A \in \M{m, n}$ where the $i$-th column is given by the coordinates of $T(v_i)\in W$ in the basis $B_W$. In symbols,
    \[
        T(v_i) = A_{1, i}w_1 + \dots + A_{m,i}w_m\mathrlap{\qquad \text{for }i = 1, \dots, n.}
    \]
\end{definition}

Matrices are representations of linear maps, and any $n$-dimensional real vector space looks exactly the same as $\RR^n$, so it looks like we are saying that all these things are the same? Well, not quite, as there is a subtle yet fundamental difference:
\begin{center}
    $\RR^n$ is an abstract vector space $V$ of dimension $n$ \textbf{together with a choice of an ordered basis}.
\end{center}
The point here is that, given a finite-dimensional abstract vector space, there is no canonical choice of a basis.
It is for this reason that the approach to linear algebra that simply talks about vector spaces and linear maps and not about matrices and $\RR^n$ is often referred to as \textbf{coordinate-free}.
There are times, as it happened in~\Cref{example:polynomials}, that there exists a clear choice, but others there is not distinguished basis.

We will spend the rest of these notes showing how thinking about a map $T$ being linear if
\[
    T(u+v) = T(u) + T(v)\qquad \text{and}\qquad T(au) = aT(u)\mathrlap{\qquad\forall u,v\in V, a\in\RR.}
\]
makes computations easier than thinking about $T$ as being a matrix. We show an example as a taster:
\begin{example}[Linear maps are simpler than matrices]
    If we think of a linear map as being a matrix, it might not be clear at first sight that the trace of a matrix $\deffun{\tr : \M{n} -> \RR;}$ is a linear map. Now, when we look at it using its abstract definition (\Cref{def:linear_map}) this becomes obvious as
    \[
        \tr(A + B) = \tr(A) + \tr(B) \qquad \tr(cA) = c\tr(A)\mathrlap{\qquad \forall A,B \in \M{n}, c\in \RR.}
    \]
    Transposing a matrix $A \mapsto \trans{A}$ or taking the first two columns of a matrix are other examples of linear maps on matrices. We will see many more in~\Cref{sec:forward_mode}.
\end{example}
\section{Multivariable Calculus}\label{sec:multivariable_calculus}
In this section, we go over the standard definition of a differentiable map from multivariable calculus. We will avoid the language of coordinates---\eg, Jacobians and partial derivatives---whenever possible.

\begin{definition}\label{def:differential}
    A map $\deffun{f : \RR^m -> \RR^n;}$ is \textbf{differentiable} at a point $x \in \RR^m$ if there exists a linear map $\deffun{L_x : \RR^m -> \RR^n;}$ such that
    \[
        \lim_{h \to 0_{\RR^m}}\frac{f(x+h) - f(x) - L_x(h)}{\norm{h}_{\RR^m}} = 0_{\RR^n}.
    \]
    In this case, we denote the linear map $\pa{df}_x = L_x$, and we call it the \textbf{differential of $f$ at $x$}. We say that $\pa{df}_x(v)$ is the \textbf{directional derivative of $f$ at $x$ in the direction $v$}.
\end{definition}

\begin{remark}[The differential is a first order linear approximation]
All this definition is saying is that there exists a linear first order approximation of $f$ at $x$. In other words, if we subtract the approximation from $f$, what we have left vanishes at zero slower than linearly---\ie, we have removed all the ``linear terms'' of $f$. Another way to look at this is by thinking that $\pa{\dif f}_x$ is the first term of the Taylor expansion of $f$
\begin{equation}\label{eq:taylor_expansion}
    f(x+h) = f(x) + \pa{\dif f}_x(h) + o(h)
\end{equation}
where $f(x)$ is the $0^{\text{th}}$ order approximation to $f$ at $x$, $\pa{\dif f}_x(h)$ is the $1^{\text{st}}$ order approximation and the $o(h)$ denotes that the difference vanishes at zero slower than linearly.\footnote{Formally, $o(h)$ means that there exists a function $E$ such that
$f(x+h) = f(x) + \pa{\dif f}_x(h) + E(h)\norm{h}$ with $\lim_{h \to 0}E(h) = 0$. This is clear choosing $E(h) = \frac{f(x+h) -f(x) - \pa{\dif f}_x(h)}{\norm{h}}$.}
\end{remark}

Even though it is possible to generalise some of these concepts to deal with functions that are not differentiable such as $\relu(x) = \max(x, 0)$,
\begin{center}
We will always assume that the maps we consider are \textbf{differentiable}.
\end{center}

\begin{remark}[Directional derivatives]
    We defined the directional derivative of a differentiable function $f$ at $x$ in the direction $v$ as $\pa{\dif f}_x(v)$. Now, since the function is differentiable, the limit when $h$ approaches zero always exist, so we can give a much more reasonable definition of directional derivative. We can let $h$ tend to zero in the direction of $v \in \RR^m$ with $\norm{v} = 1$, by letting $h = \epsilon v$ for $\epsilon \in \RR$, getting that
    \[
        \pa{\dif f}_x(v) = \deriv{\epsilon}\Big\vert_{\epsilon=0} f(x+\epsilon v)
    \]
    where $\deriv{\epsilon}\Big\vert_{\epsilon=0}$ differentiates $f(x+\epsilon v)\in\RR^n$ coordinate-wise at $\epsilon=0$. This shows how $\pa{\dif f}_x(v) \in \RR^n$ represents how $f$ varies when $x$ is approached in the direction $v$.
\end{remark}

\begin{remark}[The matrix associated to the differential is the Jacobian]
    The differential of $f$ is a linear map $\deffun{\pa{\dif f}_x : \RR^m -> \RR^n;}$ and as such, it has a matrix representation (\Cref{def:matrix}). The columns of this matrix representation of $\pa{\dif f}_x$ are given by evaluating $\pa{\dif f}_x$ on the vectors $\set{e_i}$ of the basis of $\RR^n$
    \[
        \pa{\dif f}_x(e_i) = \deriv{\epsilon}\Big\vert_{\epsilon=0} f(x+\epsilon e_i) = \frac{\partial f}{\partial x_i}(x).
    \]
    In other words, the matrix associated to the differential is the usual Jacobian matrix.
\end{remark}

It will come to no surprise that we will never talk about partial derivatives nor use the Jacobian. We will just use the definition of the differential together with two results. The first one is the chain rule:
\begin{theorem}[Chain rule]\label{thm:chain_rule}
    Let $\deffun{f : \RR^m -> \RR^n;}$ and $\deffun{g : \RR^n -> \RR^p;}$ be two differentiable maps. We have that
    \[
        \dif \pa{g \circ f}_x = \pa{\dif g}_{f(x)} \circ \pa{\dif f}_x\mathrlap{\qquad\forall x \in \RR^m}
    \]
    or equivalently
    \[
        \dif \pa{g \circ f}_x\pa{v} = \pa{\dif g}_{f(x)}\cor{\pa{\dif f}_x\pa{v}}\mathrlap{\qquad \forall x,v\in \RR^m.}
    \]
\end{theorem}
This formula is the crux of all \ad{} engines, as we will see in the sequel. Note that this formula defines the equality between a linear map and the composition of two other linear maps. It is also worth noting that the domains and codomains of these maps are compatible:

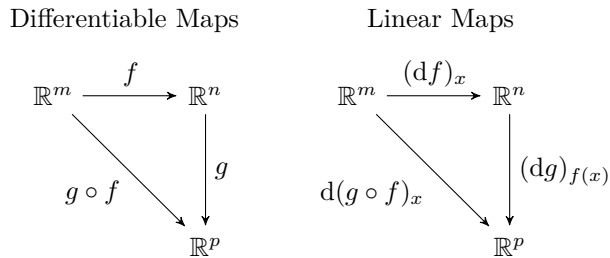
\begin{center}
	\begin{tikzpicture}
        \node[text width=4cm, anchor=west, right] at (-0.7,1) {Differentiable Maps};
		\node (A) {$\RR^m$};
		\node (B) [right of=A] {$\RR^n$};
		\node (C) [below of=B] {$\RR^p$};
		\draw[->] (A) to node {$f$} (B);
		\draw[->] (B) to node {$g$} (C);
        \draw[->] (A) to node [swap] {$g\circ f$} (C);
        \node[text width=4cm, anchor=west, right] at (4,1) {Linear Maps};
        \node (A1) [right of=B] {$\RR^m$};
		\node (B1) [right of=A1] {$\RR^n$};
		\node (C1) [below of=B1] {$\RR^p$};
        \draw[->] (A1) to node {$\pa{\dif f}_x$} (B1);
        \draw[->] (B1) to node {$\pa{\dif g}_{f(x)}$} (C1);
        \draw[->] (A1) to node [swap] {$\dif \pa{g\circ f}_x$} (C1);
	\end{tikzpicture}
\end{center}

The other rule that we will use repeatedly is a more abstract version of the derivative of the product. This roughly says that ``differentiating a function of two variables accounts for differentiating the first variable fixing the second one plus differentiating the second variable fixing he first one''.

\begin{proposition}[Leibnitz rule]\label{prop:leibnitz}
    Let $\deffun{f : \RR^n \times \RR^m -> \RR^p;}$ be a differentiable function. Define $\deffun{f_{1,x} : \RR^m -> \RR^p;}$ for $x \in \RR^n$ as $f$ partially evaluated on the first variable on $x \in \RR^n$, that is, $f_{1,x}(y) \defi f(x, y)$. Define also $f_{2, y}(x) = f(x,y)$. We have that
    \[
        \pa{\dif f}_{(x,y)}\pa{e_1, e_2} = \pa{\dif f_{2,y}}_x(e_1) + \pa{\dif f_{1,x}}_y(e_2)
    \]
\end{proposition}

The formula in this proposition may look quite difficult to parse, but it will be much easier to understand when we use it in the next section to compute the differential of some matrix functions.

\begin{remark}[On gradients]
    Note that we have not defined the gradient of a function yet, but just its differential. The definition of the gradient of a function will have to wait until~\Cref{sec:backward_mode}.
\end{remark}

\section{Forward Mode \ad}\label{sec:forward_mode}
\subsection{The model and definition}
After introducing the necessary concepts from linear algebra and multivariate calculus, we are ready to put them to good use in the context of \ad. Let us consider a model with two differentiable maps
\begin{figure}[H]
\centering
\begin{tikzpicture}
    \node (A) {$\RR^m$};
    \node (B) [right of=A] {$\RR^n$};
    \node (C) [right of=B] {$\RR^p$};
    \draw[->] (A) to node {$g$} (B);
    \draw[->] (B) to node {$f$} (C);
\end{tikzpicture}
\caption{The automatic differentiation model.}\label{fig:model}
\end{figure}

\begin{remark}[Several inputs and outputs]
    The first simplification that we have performed in~\Cref{fig:model} is that, for a map $\deffun{h : \RR^n \times \RR^m -> \RR^p;}$, we can always see it as a map from $\RR^{n + m}$ into $\RR^p$. As such, it is enough to talk about maps with just one input.\footnote{Sometimes, it may be beneficial to look at maps of several inputs though, as some properties such as linearity are not preserved by this transformation. We will see examples of this later in this section.} The same happens with maps with several outputs, only that in this case we can see them as separate maps of one output, and we can differentiate them separately.
\end{remark}

\begin{remark}[Reduction from general \ad]
    When we do automatic differentiation, we first form a directed acyclic graph of dependencies between maps---the model---ending in a map that outputs a real number---the loss function. Without entering into the details of how to define this graph of dependencies, we can think of $f$ as the last step of the graph (in topological order) that produces the real number---\ie, set $p = 1$ above---and let $g$ be the rest of the graph.
    Then, if we want to split $g$ further, we can consider its last step $g_2$ and the rest of the steps $g_1$, write $g = g_1 \circ g_2$ and proceed inductively, as we are in the same situation as above.
\end{remark}

\begin{remark}[Neural networks]
    In the case of neural networks, we have a function $F(\theta, x)$ where $\theta$ are its parameters concatenated and $x$ is an example from our dataset. We then want to differentiate with respect to the parameters. This is exactly the same idea as above, where we have the function $\theta \mapsto F(\theta, x)$ for a fixed $x$.

    Of course, this example does not encode \emph{all} possible neural network architectures, as we have not mentioned what to do when the functions involved are not differentiable, or when we use integers or have \code{if}-\code{else} constructions. That being said, extending the theory presented here to all these ideas is not formally challenging; the complexity lies in the implementation of the resulting algorithm.
\end{remark}

Now that we have everything set, we are ready to define forward mode automatic differentiation.
\begin{definition}
    \textbf{Forward mode \ad} for the model represented in~\Cref{fig:model} accounts for computing
    \[
        \dif\pa{g \circ f}_x\mathrlap{\qquad \text{for } x \in \RR^m.}
    \]
\end{definition}

In plain words, forward mode \ad{} computes the differentials of the model with respect to the parameters. This may be achieved incrementally via the chain rule (\Cref{thm:chain_rule}), as it tells us how to put together the differential of two functions to compute the differential of the composition. These ideas are often presented in the literature as ``forward mode \ad{} accumulates the product of the Jacobians'', which is the same idea but in coordinates.

\begin{remark}[Dual numbers]
    Forward mode \ad{} is frequently defined in terms of \textbf{dual numbers}. Dual numbers are defined through an abstract quantity called $\epsilon$ with the property that $\epsilon^2 = 0$. Points are then described as $v+\dot{v}\epsilon$ and the following expansion is stated
    \begin{equation}\label{eq:dual_numbers}
        f\pa{v+\dot{v}\epsilon} \stackrel{?}{=} f(v) + f'(v)\dot{v}\epsilon.
    \end{equation}
    Now, this second equality is often not justified, and left to the reader to interpret.\footnote{This is not really true. These identities can be formalised via \emph{perturbation theory}~\parencite{kato1995perturbation}.} It is typically shown by means of the computation of the derivative of the product of real numbers.
    \[
        \pa{v+\dot{v}\epsilon}\pa{u+\dot{u}\epsilon} = uv + (\dot{u}v+u\dot{v})\epsilon\mathrlap{\qquad u,v,\dot{u},\dot{v}\in \RR.}
    \]
    A moment's reflection shows that all this approach is encoding is the idea that the differential is a first order approximation to the function! Having an $\epsilon$ such that $\epsilon^2 = 0$ simply says that we just care about the terms that are linear in epsilon, and we discard any term of order two or higher.
    This is exactly what the idea of the differential of a map formalises. As such, a formal justification of~\eqref{eq:dual_numbers} is then given by the order $1$ Taylor expansion in~\eqref{eq:taylor_expansion}.

    The example of the multiplication of real numbers can then be described in the language of calculus as letting $\deffun{f : \RR^2 -> \RR;}$ defined by $f(x,y) = xy$ and computing $\pa{\dif f}_{\pa{u, v}}\pa{\dot{u}, \dot{v}} = \dot{u}v+u\dot{v}$.\footnote{We will prove this formula and generalise it to vectors and matrices later in this section.}
\end{remark}

\subsection{Computing forward mode \ad}
We will spend the rest of this section showing how the abstract definition of the differential (\Cref{def:differential}) makes computations surprisingly easy.

\begin{proposition}[Differential of a linear map]\label{prop:dif_linear_map}
    Let $\deffun{T: \RR^m -> \RR^n;}$ be a linear map, we have that
    \[
        \pa{\dif T}_x(v) = T(v)\mathrlap{\qquad \forall x,v \in \RR^m.}
    \]
\end{proposition}
\begin{proof}
    Plugging $\pa{\dif T}_x(v) = T(v)$ into the definition of the differential we get
    \[
        \lim_{h \to 0_{\RR^m}}\frac{T(x+h) - T(x) - T(h)}{\norm{h}_{\RR^m}} =
        \lim_{h \to 0_{\RR^m}}\frac{T(x) + T(h) - T(x) - T(h)}{\norm{h}_{\RR^m}} =
        \lim_{h \to 0_{\RR^m}}\frac{0_{\RR^n}}{\norm{h}_{\RR^m}} =
        0_{\RR^n}.\qedhere
    \]
\end{proof}

We can use this to compute the differential of a number of functions widely used in machine learning.\vspace{-0.5em}
\begin{example}[Differential of a linear layer. Trailing batch dimension]
    Fix $X \in \M{m, b}$ a batch of $b$ vectors of size $m$ (trailing batch dimension), and let $A \in \M{n,m}$. We can define
    \[
        \deffun{f : \M{n, m} -> \M{n, b}; A -> AX}
        \mathrlap{\qquad\qquad X \in \M{m, b}.}
    \]
    This is just the usual definition of a linear layer depending on the parameter and with fixed inputs. We do this because we want to differentiate with respect to the parameters. It is clear that
    \[
        f(A+B) = f(A) + f(B) \qquad f(cA) = cf(A) \mathrlap{\qquad \forall A,B \in \M{n,m}, c \in \RR.}
    \]
    so $f$ is linear and by~\Cref{prop:dif_linear_map}
    \[
        \pa{\dif f}_A(E) = EX\mathrlap{\qquad{\text{for }E \in \M{n, m}}.}
    \]
\end{example}

\begin{example}[Differential of a linear layer. Front batch dimension]
    The batch dimension in machine learning is often the first dimension of the tensor due to the layout of matrices in memory. For this reason, it is common to write a batch $b$ of $m$-dimensional vectors as $X \in \M{b, m}$. In this case, we may write a linear layer as
    \[
        \deffun{f : \M{n, m} -> \M{b, n}; A -> X\trans{A}}
        \mathrlap{\qquad\qquad X \in \M{b, m}}
    \]
    This function has a transpose and a matrix multiplication. In particular, if we write $g(A) = \trans{A}$ and $h(A) = XA$ we have that $f = h \circ g$.
    Luckily, both the transpose and the matrix multiplication are linear functions, so we can use the chain rule (\Cref{thm:chain_rule}) and the formula for the differential of a linear function twice to compute the differential of $f$:
    \[
        \pa{\dif f}_A(E) =
        \dif\pa{h \circ g}_A(E) =
        \cor{\pa{\dif h}_{g(A)} \circ \pa{\dif g}_A}(E) =
        \pa{\dif h}_{\trans{A}}(g(E)) =
        h(\trans{E}) =
        X\trans{E}.
    \]
    A simpler way of performing this computation is by noting that $f$ is linear itself so
    \[
        \pa{\dif f}_A(E) = f(E) = X\trans{E}.
    \]
\end{example}

\begin{example}[More linear maps]
    Linear maps come in different shapes and forms
    \begin{itemize}
        \item \textbf{Inner product of vectors}.
            Let $\deffun{f_{2,y} : \RR^n -> \RR;}$ with $f_{2,y}(x) = \trans{x}y = \sum_{i=1}^n x_i y_i$ for a fixed $y \in \RR^n$, then $\pa{\dif f_{2,y}}_x(v) = \trans{v}y$. Fixing the first variable, $f_{1,x}(y) = \trans{x}y$, $\pa{\dif f_{1,x}}_y(v) = \trans{x}v$.
        \item \textbf{Trace of a matrix}.
            Let $\deffun{f : \M{n} -> \RR;}$ with $f(A) = \tr\pa{A}$, then $\pa{\dif f}_A(E) = \tr\pa{E}$.
        \item \textbf{Inner product of matrices}.
            Let $\deffun{f_{1,X} : \M{m,n} -> \RR;}$ with $f_{1,X}(A) = \tr\pa{\trans{A}X} = \sum_{i=1}^m\sum_{j=1}^n A_{ij} X_{ij}$ for a fixed $X \in \M{m,n}$. \footnote{Note that $\tr\pa{\trans{A}X}$ is just a convenient way to represent the inner product of matrices as seen as vectors of size $\RR^N$ with $N = mn$.}
            We then have $\pa{\dif f_{1,X}}_A(E) = \tr\pa{\trans{E}X}$ and an analogous formula for the second variable.
    \end{itemize}
\end{example}

\begin{example}[Several inputs]
    Consider the inner product of vectors as a function of two arguments
    \[
        \deffun{f : \RR^n \times \RR^n -> \RR; x, y -> \trans{x}y;}
    \]
    and define the function partially evaluated in its first and second argument as $\deffun{f_{i,x} : \RR^n -> \RR;}$ for $i = 1,2$ so that $f_{1,x}(y) = f_{2,y}(x) = f(x,y)$. We have that $f_{i,x}$ are linear, as for $i=1,2$
    \[
    f_{i,u}\pa{v+w} = f_{i,u}\pa{v} + f_{i,u}\pa{w} \qquad f_{i,u}\pa{cv} = cf_{i,u}\pa{v}\qquad\qquad \forall u,v,w \in \RR^n, c \in \RR.
    \]
    For this reason, we can compute its differential using Leibnitz rule (\Cref{prop:leibnitz})
    \[
        \pa{\dif f}_{(x,y)}\pa{e_1, e_2} = \pa{\dif f_{2,y}}_x(e_1) + \pa{\dif f_{1,x}}_y(e_2) = f_{2,y}(e_1) + f_{1,x}(e_2) = \trans{e_1}y + \trans{x}e_2\qquad x, y, e_1, e_2 \in \RR^n.
    \]
    In contrast, note that $f$ itself is not linear as a function from $\RR^{2n}$ to $\RR$ as $f(ax,ay) = a^2f(x,y)$.
\end{example}

\begin{example}[Powers of a matrix]\label{ex:powers}
    Consider the map that multiplies a matrix with itself $k$ times
    \[
        \deffun{f : \M{n} -> \M{n}; A -> A^k = A \stackrel{k)}{\cdots} A}
    \]
    this is the same as evaluating the map $g(A_1, \dots, A_k) = A_1 \dots A_k$ for $A_i \in \M{n}$ at $(A, A, \dots, A)$. We can compute the differential of $g$ using~\Cref{prop:leibnitz} since $g$ is linear in every entry:
    \[
        \pa{\dif g}_{\pa{A_1, \dots, A_k}}\pa{E_1, \dots, E_k} = E_1A_2 \cdots A_k + A_1E_2 \cdots A_k + \dots + A_1A_2\cdots E_k.
    \]
    So the differential of $f$ is given by
    \[
        \pa{\dif f}_A(E)
        = \pa{\dif g}_{\pa{A, \dots, A}}\pa{E, \dots, E}
        = EA^{k-1} + AEA^{k-2} + \dots + A^{k-1}E
        = \sum_{i=0}^{k-1} A^iEA^{k-i-1}.
    \]
\end{example}

    This example shows that, morally, if we can write a map $f$ as a map $g$ on more variables such that $g$ is linear in each of its variables, all we need to do to compute the differential of $f$ is to substitute each appearance of $A$ by $E$ on $f$ and add them all together. More generally, if the function is not linear in some of the variables, we substitute every appearance of $g(A)$ by $\pa{\dif g}_A(E)$ as described in~\Cref{prop:leibnitz}. We show this idea in the following example.

\begin{example}[Matrix inverse]\label{ex:inverse}
    Let $\GL{n} \subset \M{n}$ be the set of invertible matrices. Define
    \[
        \deffun{f : \GL{n} -> \GL{n}; A -> A^{-1}}
    \]
    We have that, by definition of the matrix inverse
    \[
        Af(A) = \I_n\mathrlap{\qquad \forall A \in \GL{n}.}
    \]
    Defining $g(A) = Af(A)$ this identity can be rewritten as $g(A) = \I_n$ for $A \in \GL{n}$.
    This is an equality between functions---one of them constant---so we may differentiate them.\footnote{Formally, we would first need to define what does it mean to differentiate over $\GL{n}$. Luckily, $\GL{n}$ is an open subset of $\M{n}$, and since the definition of differential is local, we can always define the differential at any matrix $A \in\GL{n}$ by restricting the limit in~\Cref{def:differential} to a neighbourhood of $A$.}
    It is direct to see from~\Cref{def:differential} that the differential of a constant map is the function that maps any $E$ to the zero matrix. On the left-hand side we apply~\Cref{prop:leibnitz} to get
    \[
        EA^{-1} + A\pa{\dif f}_A(E) = 0_{n\times n}
    \]
    and solving for $\pa{\dif f}_A(E)$ we get
    \[
        \pa{\dif f}_A(E) = -A^{-1}EA^{-1}.
    \]
    Note that this is a far-reaching generalisation of the result $(1/x)' = -1/x^2$ for $x \in \RR\backslash\set{0} = \GL{1}$.
\end{example}

Before giving the last result, we show how to extend functions on the real numbers to matrix functions.

\begin{definition}\label{def:matrix_function}
    Let $f_{\RR}(x) = \sum_{k=0}^\infty c_kx^k$ be an analytic function---\ie, a function equal to its Taylor series. We define its associated \textbf{matrix function} as\footnote{A matrix function is defined at $A \in \M{n}$ if and only if all the eigenvalues of $A$ lie in the domain of definition of $f$ when seen as a function from $\CC$ to $\CC$.}
    \[
        \deffun{f : \M{n} -> \M{n} ; A -> \sum_{k=0}^\infty c_k A^k.}
    \]
\end{definition}

\begin{example} Any function with a Taylor series can be turned into a matrix function:

    \vspace{\baselineskip}
    \noindent
    \begin{minipage}[t]{.5\textwidth}
    \begin{itemize}
        \item \textbf{Exponential}: $\exp(A) = \sum_{k=0}^\infty \frac{1}{k!}A^k$
        \item \textbf{Logarithm}: $\log\pa{I_n + A} = \sum_{k=0}^\infty \frac{(-1)^k}{k+1}A^k$
    \end{itemize}
    \end{minipage}%
    \begin{minipage}[t]{.5\textwidth}
    \begin{itemize}
        \item \textbf{Sine}: $\sin(A) = \sum_{k=0}^\infty \frac{(-1)^k}{(2k+1)!}A^{2k+1}$
        \item \textbf{Cosine}: $\cos(A) = \sum_{k=0}^\infty \frac{(-1)^k}{(2k)!}A^{2k}$
    \end{itemize}
    \end{minipage}
    \\[3pt]
\end{example}

We present the last and most general result of this section, which can be roughly summarised as:
\begin{center}
    If we know how to approximate a matrix function, we know how to approximate its differential.
\end{center}

\begin{theorem}[Differential a Matrix Function {\parencite{mathias1996chain}}]\label{thm:mathias}
    Let $\deffun{f : \M{n} -> \M{n};}$ be a matrix function (\Cref{def:matrix_function}). Applying $f$ on a matrix of size $2n \times 2n$, we get the following result by blocks:
    \[
        f\begin{pmatrix}
            A & E \\
            0 & A
        \end{pmatrix} =
        \begin{pmatrix}
            f(A) & \pa{\dif f}_A(E) \\
            0 & f(A)
        \end{pmatrix}\mathrlap{\qquad \forall A,E\in\M{n}.}
    \]
\end{theorem}
\begin{proof}
    Differentiating the series term by term and using~\Cref{ex:powers} we have that
    \begin{equation}\label{eq:differential_analytic}
        \pa{\dif f}_A(E) = \sum_{k=0}^\infty c_k\sum_{i=0}^{k-1}A^iEA^{k-i-1}.
    \end{equation}
    We can also compute the powers of the block matrix
    \[
        \begin{pmatrix}
            A & E \\
            0 & A
        \end{pmatrix}^k =
        \begin{pmatrix}
            A^k & \sum_{i=0}^{k-1}A^iEA^{k-i-1} \\
            0 & A^k
        \end{pmatrix}
    \]
    so
    \[
        f\begin{pmatrix}
            A & E \\
            0 & A
        \end{pmatrix} =
        \sum_{k=0}^\infty c_k
        \begin{pmatrix}
            A & E \\
            0 & A
        \end{pmatrix}^k =
        \sum_{k=0}^\infty
        \begin{pmatrix}
            c_k A^k & c_k \sum_{i=0}^{k-1}A^iEA^{k-i-1} \\
            0 & c_k A^k
        \end{pmatrix} =
        \begin{pmatrix}
            f(A) & \pa{\dif f}_A(E) \\
            0 & f(A)
        \end{pmatrix}.\qedhere
    \]
\end{proof}

\section{Backward Mode \ad}\label{sec:backward_mode}
In this section, we go over the most popular method of automatic differentiation: Backward \ad. This method has the advantage that, in order to compute the backward pass for a model, one does not have to deal with matrices---the Jacobians associated to the differential of the model---but just with vectors of the same size of the parameters.

\subsection{Inner products, gradients, and adjoints}
Before introducing the definition of backward mode \ad, we need to define one more mathematical concept from linear algebra.
\begin{definition}\label{def:inner_product}
    Let $V$ be a real vector space, a (real) \textbf{inner product} is a map
    \[
        \deffun{\scalar{-,-} : V \times V -> \RR; x, y -> \scalar{x,y}}
    \]
    such that it is
    \begin{itemize}
        \item Bilinear: It is linear in each variable.
        \item Symmetric: $\scalar{x,y} = \scalar{y,x}$ for every $x,y \in V$.
        \item Positive definite: $\scalar{x,x} > 0$ for every $x \in V$, $x \neq 0$.
    \end{itemize}
\end{definition}

We will write \spd{} as short for symmetric positive definite matrix, as we will use them in examples.

\begin{example}\label{ex:inner_prod}
    The following are examples of inner products
    \begin{itemize}
        \item \textbf{Canonical inner product on $\RR^n$:} $\scalar{x,y} = \trans{x}y$ for $x,y \in \RR^n$.
        \item \textbf{Other inner products on $\RR^n$:} $\scalar{x,y} = \trans{x}Hy$ for $x,y \in \RR^n$ and a fixed $H \in \M{n}$ \spd.\footnote{To prove that this is positive definite, consider the Cholesky decomposition of $H = \trans{U}U$ with $U$ upper-triangular.}
        \item \textbf{Canonical inner product on $\M{m,n}$:} $\scalar{A,B} = \tr\pa{\trans{A}B}$ for $A,B \in \M{m,n}$.
        \item \textbf{Other inner products on $\M{m,n}$:} $\scalar{A,B} = \tr\pa{\trans{A}HB}$ for $A,B \in \M{m,n}$ and $H\in\M{m}$ \spd.
    \end{itemize}
\end{example}

Inner products allow us to measure norms of vectors $\norm{x} = \sqrt{\scalar{x, x}}$, angles between vectors $\angle\pa{x,y} = \arccos\frac{\scalar{x,y}}{\norm{x}\norm{y}}$, distances $d(x,y) = \norm{x-y}$, and many other metric properties. As such, it will come to no surprise the fact that inner products are very important in machine learning and optimisation. For one, we need them to talk about the distance from a point to the optimum and rates of convergence. Perhaps less known is the fact that we also require them to talk about gradients.

\begin{remark}[Motivating the concept of gradient]
For a function $\deffun{f : \RR^n -> \RR;}$ and an $x \in \RR^n$, the map $v \mapsto \pa{\dif f}_x(v)$ is a linear function from $\RR^n$ to $\RR$. Now, if we have an inner product $\scalar{-,-}$ on $\RR^n$, for a fixed $g \in \RR^n$, the function $v \mapsto \scalar{g, v}$ is also a linear function from $\RR^n$ to $\RR$. The question now is, given an inner product on $\RR^n$ and a function $f$, can we always represent the differential of $f$ as a vector $g_x \in \RR^n$ such that $\pa{\dif f}_x(v) = \scalar{g_x, v}$? This is, in fact, the case, and it is the definition of a well-known concept.
\end{remark}

\begin{definition}
    Let $\deffun{f : \RR^n -> \RR;}$, and let $\scalar{-, -}$ be an inner product on $\RR^n$. We define the \textbf{gradient of $f$ at $x\in \RR^n$} as the vector $\grad f(x) \in \RR^n$ such that
    \[
        \pa{\dif f}_x(v) = \scalar{\grad f(x), v}\mathrlap{\qquad \forall v \in\RR^n.}
    \]
\end{definition}

\begin{remark} A number of remarks are in order.
    \begin{itemize}
        \item As $x \mapsto \trans{v}x$ for $x,v \in \RR^n$ is a linear function, some people like to think informally of vectors as ``column vectors'' and linear functions as ``row vectors''. This way, the operation of going from a differential $\trans{v} \cdot -$ to a gradient $v$ for the canonical inner product looks like ``transposing'' $\trans{v}$.
        \item \textbf{Important}. The gradient and the differential of a function are not the same thing. The first one is a \textbf{function} into the real numbers, while the latter one is a \textbf{vector}.
        \item The gradient of a function depends on the choice of inner product, the differential on finite-dimensional spaces does not, since all the norms are equivalent.
    \end{itemize}
\end{remark}

\begin{example}
    We compute the gradient of some functions building on results from~\Cref{sec:forward_mode}.
    \begin{itemize}
        \item Consider $\RR^n$ with the canonical inner product, and let $f(x) = \scalar{g, x} = \trans{g}x$ for a fixed $g\in \RR^n$. Since $f$ is linear $\pa{\dif f}_x(v) = \scalar{g, v}$, and by definition of a gradient, $\grad f(x) = g$ for every $x\in\RR^n$.
        \item Consider $\RR^n$ with an arbitrary inner product, and let $f(x) = \scalar{g,x}$ then $\grad f(x) = g$ for every $x\in\RR^n$.
        \item Consider $\RR^n$ with the inner product $\scalar{u,v} = \trans{u}Hv$ for $H$ \spd{} (see \Cref{ex:inner_prod}) and let $f(x) = \trans{x}y$ for a fixed $y \in \RR^n$. As in the first example, $\pa{\dif f}_x(v) = \trans{v}y$ since $f$ is linear. On the other hand
            \[
            \pa{\dif f}_x(v)
            = \trans{v}y
            = \trans{v}HH^{-1}y
            = \scalar{v, H^{-1}y}
            = \scalar{H^{-1}y, v}
            \]
            and so $\grad f(x) = H^{-1}y$ for every $x \in \RR^n$.
        \item Consider $\M{n}$ with the canonical inner product $\scalar{A,B} = \tr\pa{\trans{A}B}$, and let $f(A) = \tr(A)$. Since $f$ is linear $\pa{\dif f}_A(E) = \tr(E) = \tr\pa{\trans{\pa{\I_n}}E} = \scalar{\I_n, E}$. Thus, $\grad f(A) = \I_n$ for every $A \in \M{n}$.
        \item Let $\deffun{f : \RR^n -> \RR;}$, and consider the canonical inner product on $\RR^n$. The $i$-th coordinate of $\grad f(x)$ is equal to $\frac{\partial f}{\partial x_i}(x)$.
    \item Let $\deffun{f : \RR^n -> \RR;}$, and consider the inner product on $\RR^n$ given by $\scalar{x,y} = \trans{x}Hy$ for $H \in \M{n}$ \spd. Denote by $g_x \in \RR^n$ the vector with $i$-th coordinate equal to $\frac{\partial f}{\partial x_i}(x)$---\ie, the gradient of $f$ with respect to the canonical inner product. Then $\grad f(x) = H^{-1}g_x$, while $\pa{\dif f}_x(v) = \trans{g_x}v$ regardless of the inner product.
    \end{itemize}
\end{example}

\begin{remark}[Gradient of a composition]
    Consider a function $\deffun{h : \RR^m -> \RR;}$ defined as a composition $h = f \circ g$ with $\deffun{g : \RR^m -> \RR^n;}$ and $\deffun{f : \RR^n -> \RR;}$ and fix inner products on $\RR^m$ and $\RR^n$. How do we compute the gradient of $h$ in terms of $g$ and $f$? By the chain rule (\Cref{thm:chain_rule}) and the definition of the gradient of $f$ we have that for $x \in \RR^m$, denoting $y = g(x) \in \RR^n$
\[
    \pa{\dif h}_x(v) = \pa{\dif f}_y\cor{\pa{\dif g}_x(v)} = \scalar{\grad f(y), \pa{\dif g}_x(v)}.
\]
To be able to compute the gradient of $h$ at $x$, we would have to solve for $v$ on the last equality, sending the linear map $\pa{\dif g}_x$ to the left-hand side of the inner product. This is exactly what the adjoint of a linear map achieves.
\end{remark}

\begin{definition}
    Let $\deffun{T : V -> W;}$ be a linear map between real finite-dimensional vector spaces with inner products $\scalar{-, -}_V, \scalar{-,-}_W$. We define its \textbf{adjoint} as the linear map $\deffun{T^\ast : W -> V;}$ such that
    \[
        \scalar{w, T(v)}_W = \scalar{T^\ast(w), v}_V\mathrlap{\qquad \forall v\in V, w \in W.}
    \]
\end{definition}

    Before giving examples of the adjoint of some linear maps, we formalise the motivation that led to the definition of the adjoint.

    \begin{proposition}[Gradient of a composition]\label{prop:grad_composition}
        Let $\deffun{g : \RR^m -> \RR^n;}$ and $\deffun{f : \RR^n -> \RR;}$ and fix inner products on $\RR^m$ and $\RR^n$. We have that for every $x \in \RR^m$, denoting $y = g(x) \in \RR^n$,
        \[
            \grad\pa{f \circ g}(x) = \pa{\dif g}^\ast_x\pa{\grad f(y)}
        \]
    \end{proposition}
    \begin{proof}
        We finish the computation that we started before
    \[
        \dif\pa{f \circ g}_x(v) = \pa{\dif f}_{g(x)}\cor{\pa{\dif g}_x(v)} = \scalar{\grad f(y), \pa{\dif g}_x(v)}_{\RR^n} = \scalar{\pa{\dif g}^\ast_x\pa{\grad f(y)}, v}_{\RR^m}.\qedhere
    \]
    \end{proof}

    \begin{example}[Adjoint of the matrix multiplication]\label{ex:adjoint_mm}
        Consider the linear map of multiplying on the right by a matrix $\deffun{R_X : \M{m,n} -> \M{m,p};}$, $R_X(A) = AX$ for a fixed matrix $X \in \M{n,p}$. For the canonical inner products on $\M{m,n}$ and $\M{m,p}$:
        \[
            \scalar{B, R_X(A)}_{\M{m,p}}
            = \tr\pa{\trans{B}AX}
            = \tr\pa{X\trans{B}A}
            = \tr\pa{\trans{\pa{B\trans{X}}}A}
            = \scalar{B\trans{X}, A}_{\M{m,n}}
        \]
        In other words $\pa{R_X}^\ast(B) = B\trans{X} = R_{\trans{X}}(B)$, or simply $\pa{R_X}^\ast = R_{\trans{X}}$.

        An analogous computation gives that, for the left multiplication $L_X(A) = XA$ with respect to the canonical inner products, $\pa{L_X}^\ast = L_{\trans{X}}$.
    \end{example}

    \begin{example}[The adjoint depends on the choice of inner product]
        Consider $R_X(A) = AX$ for $A, X \in \M{n}$ as defined in~\Cref{ex:adjoint_mm} and consider the inner product $\scalar{A,B} = \tr\pa{\trans{A}HB}$ for a fixed $H \in \M{n}$ \spd{} (\cf, \Cref{ex:inner_prod}). We have
        \[
            \scalar{B, R_X(A)}
            = \tr\pa{\trans{B}HAX}
            = \tr\pa{\trans{\pa{B\trans{X}}}HA}
            = \scalar{R_{\trans{X}}(B), A}
        \]
        so $R^\ast_X = R_{\trans{X}}$, as before. On the other hand, for the left multiplication $L_X(A) = XA$,
        \[
            \scalar{B, L_X(A)}
            = \tr\pa{\trans{B}HXA}
            = \tr\pa{\trans{B}HXH^{-1}HA}
            = \scalar{L_{H^{-1}\trans{X}H}(B), A}
        \]
        so $L_X^\ast = L_{H^{-1}\trans{X}H}$, where we have used that the inverse of an \spd{} matrix is symmetric.
    \end{example}

    We finish this section enumerating two properties that will be particularly useful in~\Cref{sec:computing_backward_ad}.
    \begin{proposition}\label{prop:properties_adjoint}
        Let $\deffun{S : U -> V;}$ and $\deffun{T, T_1, T_2 : V -> W;}$ be linear maps between (finite-dimensional real) vector spaces with inner products, then
        \begin{itemize}
            \item The adjoint is linear. Defining $(aT)(v) = aT(v)$ and $(T_1 + T_2)(v) = T_1(v) + T_2(v)$ for $v\in V$, $a\in\RR$, then $(aT)^\ast = aT^\ast$ and $\pa{T_1 + T_2}^\ast = T_1^\ast + T_2^\ast$.
            \item The adjoint reverses the order of the composition: $(T \circ S)^\ast = S^\ast \circ T^\ast$.
        \end{itemize}
    \end{proposition}

\subsection{The model and definition}\label{sec:backward_ad}
As in the case of forward \ad, we have a model described by a composition of functions. The difference is that, in this case, the last map will be a function mapping the result into the real numbers. In machine learning this is called the \emph{loss function}.

\begin{figure}[H]
\centering
\begin{tikzpicture}
    \node (A) {$\RR^k$};
    \node (B) [right of=A]{$\RR^m$};
    \node (C) [right of=B] {$\RR^n$};
    \node (D) [right of=C] {$\RR$};
    \draw[->] (A) to node {$h$} (B);
    \draw[->] (B) to node {$g$} (C);
    \draw[->] (C) to node {$f$} (D);
\end{tikzpicture}
\caption{The automatic differentiation model.}\label{fig:model_bwd}
\end{figure}

\begin{definition}
    \textbf{Backward mode \ad} for the model represented in~\Cref{fig:model_bwd} with respect to the canonical inner product on $\RR^k$ accounts for computing the gradient
    \[
        \grad \pa{f \circ g \circ h}(x) \mathrlap{\qquad \text{for } x \in \RR^k.}
    \]
\end{definition}

All the theory explained before about linear maps, differentials, inner products, gradients, and adjoints comes together to give this particularly simple definition. Even better, we have laid out the theory in such a way that we have all the tools to compute this quantity. In~\Cref{sec:forward_mode}, we saw how to compute the differential of different maps and how to compose them together. Using~\Cref{prop:grad_composition} we can compute the gradient of the model in terms of the gradient of $f$ and the adjoints of the differentials of $g$ and $h$---for example by choosing the canonical inner product on $\RR^n$ and $\RR^m$.\footnote{It is possible to prove that the result does not depend on the choice of inner product in the intermediate spaces, the only thing that changes is the matrix representation of the functions.} Finally, we use \Cref{prop:properties_adjoint} to compute the adjoint of the composition as the reversed composition of the adjoints.\footnote{It should be clear why in some fields in applied mathematics backpropagation is called \textbf{the adjoint method}.} All this together gives, denoting $y = h(x)$ and $z = g(y) = g(h(x))$,
\[
    \grad \pa{f \circ g \circ h}(x) =
    \pa{\dif h}^\ast_x\cor{\pa{\dif g}^\ast_y\pa{\grad f(z)}}\mathrlap{\qquad \forall x \in \RR^k.}
\]

Note that the last function being applied, $f$, is the first one that we differentiate. We then pass its gradient ``backwards'' to $g$ and then $h$, hence the name of the method.

\subsection{Computing backward mode \ad}\label{sec:computing_backward_ad}
We already computed the adjoint of some linear functions in~\Cref{ex:adjoint_mm}. We now show how these formulae together with the properties from~\Cref{prop:properties_adjoint} are enough to compute the adjoint of the differentials of the maps we considered in~\Cref{sec:forward_mode}.

\begin{example}[Adjoint of the powers of a matrix]
    In~\Cref{ex:powers} we showed that for $f(A) = A^k$
    \[
        \pa{\dif f}_A(E) = \sum_{i=0}^{k-1} A^iEA^{k-i-1} = \sum_{i=0}^{k-1} L_{A^i}\pa{R_{A^{k-i-1}}\pa{E}}\mathrlap{\qquad A,E\in\M{n}}
    \]
    with $L_X(A) = XA$ and $R_X(A)=AX$ being the left and right multiplication.
    For the canonical inner product on $\M{n}$, using that the adjoint is linear (\Cref{prop:properties_adjoint}), and the formulae for the adjoint of $L$ and $R$ (\Cref{ex:adjoint_mm}) we get
    \[
    \pa{\dif f}^\ast_A
        = \pa[\Big]{\sum_{i=0}^{k-1} L_{A^i} \circ R_{A^{k-i-1}}}^\ast
        = \sum_{i=0}^{k-1} \pa{L_{A^i} \circ R_{A^{k-i-1}}}^\ast
        = \sum_{i=0}^{k-1} R_{\pa{\trans{A}}^{k-i-1}} \circ L_{\pa{\trans{A}}^i}.
    \]
    or more explicitly
    \[
        \pa{\dif f}^\ast_A(E) = \sum_{i=0}^{k-1} \pa{\trans{A}}^iE\pa{\trans{A}}^{k-i-1}.
    \]
\end{example}

\begin{example}[Adjoint of matrix function]\label{ex:adjoint_analytic}
    For a matrix function $\deffun{f : \M{n} -> \M{n};}$, $f(A) = \sum_{k=0}^\infty c_k A^k$ (\cf, \Cref{def:matrix_function})
    and the canonical inner product on $\M{n}$ we have that
    \[
        \pa{\dif f}^\ast_A = \pa{\dif f}_{\trans{A}}\mathrlap{\qquad \forall A \in \M{n}.}
    \]
    This follows from the formula for the differential of $f$ computed in~\Cref{eq:differential_analytic} and the properties of the adjoint (\Cref{prop:properties_adjoint}). As a corollary, we get that the adjoint of the differential of an analytic function on matrices can be computed by applying $f$ to a larger function, using the formula in~\Cref{thm:mathias}.
\end{example}

\subsection{Exercises}
We leave here a number of exercises to help the reader wrapping their head around the material. If you want to try just two, have a look at~\Cref{ex:rank,ex:implement}.

\begin{exercise}\label{ex:rank}
    Consider a feed-forward network $F_{A, b}(x) = \ell\pa{\sigma\pa{Ax+b}}$ for a function $\deffun{\ell : \RR^m -> \RR;}$, fixed $x \in \RR^n, A \in \M{m, n}, b \in \RR^m$ and an element-wise function $\sigma$. Show that the gradient with respect to $A$ has rank $1$. In other words, show that if $\tilde{F}(A) = F_{A, b}(x) = \ell\pa{\sigma\pa{Ax+b}}$ for fixed $x, b$, $\grad \tilde{F}(A) = u\trans{v}$ for two vectors $u \in \RR^m$ and $v \in \RR^n$

    \textbf{Hint}. What is the adjoint of the map $A \mapsto Ax+b$ for fixed $x, b$ evaluated on a vector $g \in \RR^m$?
\end{exercise}

The next two exercises look more difficult, but they follow via the same argument as the one above.

\begin{exercise}
    Same as above, but with a feed-forward network of depth $d$.
\end{exercise}

\begin{exercise}
    Same as above, but with a feed-forward network of depth $d$ and in the stochastic setting, where we define the total loss as $\widehat{F}(A) = \frac{1}{r}\sum_{i=1}^r F_{A, b}(x_i)$ for input vectors $\set{x_i}_{i=1}^r$.
\end{exercise}

\begin{exercise}\label{ex:implement}
    Implement the gradient for the two layer feedforward network in~\Cref{lst:ffn}.

    \textbf{Hint}. Name more intermediate variables in \code{FFN.forward} to be able to store them.

    \textbf{Hint}. Reverse the order of the arguments in \code{FFN.forward} for it to be easier to debug. The gradient with respect to \code{b1} is not going to be correct if the gradient with respect to \code{b2} is not correct.
\end{exercise}

\begin{lstlisting}[language=python,escapechar=|,caption={Modify this PyTorch 1.9 code so that autograd passes.},label={lst:ffn}]
import torch

class FFN(torch.autograd.Function):
    @staticmethod
    def forward(ctx, x, y, A1, b1, A2, b2):
        x = (A1 @ x + b1).sigmoid()
        x = (A2 @ x + b2).sigmoid()
        loss = (x - y).pow(2).sum()

        ctx.save_for_backward(...)
        return loss

    @staticmethod
    def backward(ctx, g_l):
        t1, t2, ... = ctx.saved_tensors
        ...
        return None, None, g_A1, g_b1, g_A2, g_b2

class Model(torch.nn.Module):
    def __init__(self, in_features, hidden_features, out_features):
        super().__init__()
        def make_param(*size):
            return torch.nn.Parameter(torch.empty(*size, dtype=torch.double))
        self.register_parameter("A1", make_param(hidden_features, in_features))
        self.register_parameter("b1", make_param(hidden_features))
        self.register_parameter("A2", make_param(out_features, hidden_features))
        self.register_parameter("b2", make_param(out_features))
        torch.nn.init.xavier_normal_(self.A1)
        torch.nn.init.xavier_normal_(self.A2)

    def forward(self, x, y):
        return FFN.apply(x, y, self.A1, self.b1, self.A2, self.b2)

x = torch.rand(32, dtype=torch.double)  # Batch size 1
y = torch.rand(8, dtype=torch.double)
model = Model(32, 16, 8)

args = (x, y, model.A1, model.b1, model.A2, model.b2)
torch.autograd.gradcheck(FFN.apply, args, atol=0.01)
\end{lstlisting}

\begin{exercise}
    Generalise your code in \code{FFN.backwards} to handle batches of arbitrary size. Then, use the code you have implemented to fit \mnist{} and feel good about yourself.
\end{exercise}

\begin{exercise}
    Compute the gradient for a recurrent neural network (\rnn) with respect to the recurrent kernel on PyTorch or just on paper.
\end{exercise}

\begin{exercise}
    Find where the adjoint for \code{matrix\_exp} is implemented in PyTorch and make sure you understand its code. \textbf{Hint}. Look for the function \code{matrix\_exp\_backward}.
\end{exercise}

\section{Complex Maps}\label{sec:complex}
\subsection{Forward mode \ad}
When we derived the formulae for forward mode \ad, they all followed from the definition of differential (\Cref{def:differential}) and the formula differential of a linear map (\Cref{prop:dif_linear_map}). As such, if we can generalise these two to complex maps, we should be able to generalise all the forward mode \ad{} to complex numbers.

In order to do this, we recall the point that we made in~\Cref{ex:complex}, $\CC^n$ is a \textbf{real} vector space of dimension $2n$. This means that for $a \in \RR$,
\[
    u+v \in \CC^n \qquad au \in \CC^n\mathrlap{\qquad\qquad \forall u,v \in \CC^n.}
\]
Furthermore, it means that these operations---again, with $a\in \RR$, \textbf{not} $a\in\CC$---satisfy all the axioms of a real vector space in~\Cref{def:vector_space}.

This real vector space structure treats the $n$ real components and $n$ imaginary components as independent, as if they were two parts of a vector of size $2n$ in $\RR^{2n} = \RR^n \times \RR^n$. As such, the norm of a vector in $\CC^n$ as a real vector space is given by
\[
    \norm{v}^2_{\CC^n} = \sum_{k=1}^n a_k^2 + b_k^2\mathrlap{\qquad \text{for }v = \pa{a_1+ib_1, \dots, a_n+ib_n}.}
\]

Using this norm, we can extend the definition of a differential of a real map to complex maps.
\begin{definition}
    A map $\deffun{f : \CC^m -> \CC^n;}$ is \textbf{real differentiable} at a point $x \in \CC^m$ if there exists a map $\deffun{\pa{\dif f}_x : \CC^m -> \CC^n;}$ which is linear over the real numbers (\cf, \Cref{def:linear_map}) such that
    \[
        \lim_{h \to 0_{\CC^m}}\frac{f(x+h) - f(x) - \pa{\dif f}_x(h)}{\norm{h}_{\CC^m}} = 0_{\CC^n}.
    \]
\end{definition}

\begin{remark}[Real differentiable vs.\ complex differentiable]
    Here we have defined the real differential as an $\RR$-linear map, that is, a map such that $\pa{\dif f}_x(av) = a\pa{\dif f}_x(v)$ for $a \in \RR$. If we require the differential to be $\CC$-linear---that is, $\pa{\dif f}_x(av) = a\pa{\dif f}_x(v)$ for $a \in \CC$---we get the definition of a complex differentiable map, often called \textbf{holomorphic map}.

    It should be clear that, if a complex map is complex differentiable, it is also real differentiable, but the opposite is not true. Consider for example $f(z) = \overline{z}$ for $z \in \CC$. We have that $f(az) = \overline{a}f(z)$ for $a \in \CC$, so it is not $\CC$-linear, but it is $\RR$-linear as $f(az) = af(z)$ for $a \in \RR$. Luckily, we will not need to use holomorphic maps, as real differentiable maps will be enough to compute differentials and gradients.
\end{remark}

The chain rule (\Cref{thm:chain_rule}) and the Leibnitz rule (\Cref{prop:leibnitz}) also hold verbatim for real differentiable maps. We also have the following equivalent to~\Cref{prop:dif_linear_map}.
\begin{proposition}[Differential of a linear map]\label{prop:dif_complex_linear_map}
    Let $\deffun{T: \CC^m -> \CC^n;}$ be an $\RR$-linear map, we have that
    \[
        \pa{\dif T}_x(v) = T(v)\mathrlap{\qquad \forall x, v \in \CC^m.}
    \]
\end{proposition}
\begin{proof}
    The proof is the same as in the real case.
\end{proof}

Having this, we can compute the differential of many maps, as we did in the real case.

\begin{example}
    We compute the differential of some linear maps from $\CC^m$ to $\CC^n$ or to $\RR^n \subset \CC^n$.
    \begin{itemize}
        \item Let $\deffun{f : \CC^n -> \RR^n;}$, $f(x) = \Im\pa{x}$ be the imaginary part of a vector. Since
            \[
                \Im(x+y) = \Im(x) + \Im(y) \qquad \Im\pa{ax} = a\Im\pa{x} \mathrlap{\qquad \forall x,y \in \CC^n, a \in \RR,}
            \]
            $f$ is $\RR$-linear and $\pa{\dif f}_x(v) = \Im\pa{v}$ for $v \in \CC^n$. Note that $f$ is \textbf{not} $\CC$-linear.
        \item Analogously, if $\deffun{f : \CC^n -> \RR^n;}$, $f(x) = \Re\pa{x}$, $f$ is $\RR$-linear and $\pa{\dif f}_x(v) = \Re\pa{v}$ for $v \in \CC^n$.
        \item Let $L_X(A) = XA$ for $X \in \MC{m,n}$, $A \in \MC{n, k}$. Since $L_X$ is $\CC$-linear, it is in particular $\RR$-linear, so $\pa{\dif L_X}_A(E) = XE$ for $E \in \MC{n,k}$.
        \item Let $f_X(A) = \transc{A}X$ for $X \in \MC{m,n}$, $A \in \MC{k, n}$ where $\transc{A} = \trans{\overline{A}}$. Note that this is \textbf{not} a $\CC$-linear map as $f_X(cA) = \overline{c}f_X(A)$ for $c \in \CC$, but it is an $\RR$-linear map, and as such, $\pa{\dif f_X}_A(E) = \transc{E}X$.
        \item The formulae for the differential of the powers of a matrix (\Cref{ex:powers}), inverse of a matrix (\Cref{ex:inverse}) and differential of a matrix function (\Cref{thm:mathias}) are also valid for complex matrices.
    \end{itemize}
\end{example}

These examples show that \textbf{formulae for forward \ad{} for complex maps are the same as their real counterparts}, as the basic formulae (\Cref{prop:dif_complex_linear_map}, chain rule, and Leibnitz rule) are the same.

\subsection{Backward mode \ad}

For backward mode \ad, all we need is a real inner product (\cf, \Cref{def:inner_product}). To do that all we need to do is to consider $\CC^n$ as a real vector space, as we did in the previous section.

\begin{proposition}
    The \textbf{canonical real inner product on $\CC^n$} as a real vector space for $x,y \in \CC^n$ can be written as
    \[
        \scalar{x,y}_{\CC^n} =
        \sum_{k=1}^n a_kc_k + b_kd_k =
        \Re\transc{x}y
        \mathrlap{\qquad \text{for}\quad\begin{matrix}x = \sum a_k+ib_k \\ y=\sum c_k + id_k\end{matrix}.}
    \]
    The \textbf{canonical real inner product on $\MC{m, n}$} as a real vector space for $X,Y \in \MC{m,n}$ can be written as
    \[
        \scalar{X,Y}_{\MC{m,n}} =
        \sum_{j=1,k=1}^n A_{jk}C_{jk} + B_{jk}D_{jk} =
        \Re\tr\pa{\transc{X}Y}
        \mathrlap{\qquad \text{for}\quad \begin{matrix}X = A+iB \\ Y=C+iD\end{matrix}.}
    \]
\end{proposition}
\begin{proof}
    Note that the first equality in the vector case comes the definition of the canonical real inner product on $\CC^n$, which is just the inner product on $\RR^{2n}$ (\cf, \Cref{ex:inner_prod}). Same happens for the $\MC{m,n}$ case.

    We prove this proposition for the matrix case, as vectors can be seen as the case $\MC{n,1} = \CC^n$.

    We start by rewriting the left-hand side in a coordinate-free way
    \[
        \sum_{j,k} A_{jk}C_{jk} + B_{jk}D_{jk} =
        \scalar{A,C}_{\M{m,n}} + \scalar{B,D}_{\M{m,n}}.
    \]
    Thus, we just need to prove that $\Re\tr\pa{\transc{X}Y} = \scalar{A,C}_{\M{m,n}} + \scalar{B,D}_{\M{m,n}}$, but this is direct as
    \[
        \tr\pa{\transc{X}Y}
        = \tr\pa{\pa{\trans{A}-i\trans{B}}\pa{C+iD}}
        = \underbrace{\tr\pa{\trans{A}C} + \tr\pa{\trans{B}D}}_{\text{real part}} + i\underbrace{\cor{\tr\pa{\trans{A}D} - \tr\pa{\trans{C}B}}}_{\text{imaginary part}}.
    \]
    Note that since $\scalar{A,C}_{\M{m,n}} + \scalar{B,D}_{\M{m,n}}$ is a (real) inner product,\footnote{Note that in mathematics we also find complex inner products. These are sesquilinear maps rather than bilinear. These products are more general than the real inner products, as their real part is always a real inner product, while their imaginary part is a non-degenerate symplectic (\ie, skew-symmetric) bilinear form. Luckily, we do not need these to compute gradients.}
        so is $\Re\tr\pa{\transc{X}Y}$. In other words, it is a symmetric positive definite (real) bilinear map (\Cref{def:inner_product}).
\end{proof}

All this proposition says is that $\Re\tr\pa{\transc{X}Y}$ is a convenient way to write the canonical real inner product on $\MC{m,n}$. Now, since $\MC{m,n}$ is a real vector space, and we have a real inner product on it, all the definitions and general results in~\Cref{sec:backward_mode} translate to this setting. Note that the gradients are just defined for functions with values in $\RR$, not $\CC$, while the adjoints are defined for arbitrary maps.

\begin{example}
    Consider the canonical real inner products in each of the spaces
    \begin{itemize}
        \item Let $\deffun{f : \CC^n -> \RR;}$, $f(x) = \Re\trans{g}x$ for a fixed $g\in \CC^n$. Since $f$ is $\RR$-linear $\pa{\dif f}_x(v) = \Re\trans{g}x = \scalar{\overline{g}, v}_{\CC^n}$, and by definition of a gradient, $\grad f(x) = \overline{g}$ for every $x\in\CC^n$.
        \item For $L_X(A) = XA$, $X \in \MC{m,n}, A \in \MC{n,k}$, we have that $\deffun{L_X^\ast : \MC{n,k} -> \MC{m,n};}$ is given by $L_X^\ast(E) = \transc{X}E$, so that $L_X^\ast = L_{\transc{X}}$.\footnote{This shows why some areas of mathematics abuse the notation and write $\transc{X}$ as $X^\ast$.}
        \item Similarly, if $R_X(A) = AX$ for complex $A,X$, $R_X^\ast = R_{\transc{X}}$.
        \item \Cref{ex:adjoint_analytic} translates to $\pa{\dif f}^\ast_A = \pa{\dif f}_{\transc{A}}$ for an analytic function $\deffun{f : \MC{n} ->\MC{n};}$ and $A \in \MC{n}$.
    \end{itemize}
\end{example}

In this case, the formulae are almost the same, but it tends to happen that when a matrix or a vector is transposed in the real case, it is transposed and conjugated in the complex case.

\printbibliography

\end{document}